\documentclass[5p,times]{elsarticle}

\usepackage{amssymb}
\usepackage{lipsum}
\usepackage[utf8]{inputenc}
\usepackage{textcomp}
\usepackage{bm}
\usepackage{amsmath}

\usepackage[version=4]{mhchem}
\usepackage{siunitx}
\usepackage{longtable,tabularx}
\usepackage{color}
\usepackage{amsthm}

\theoremstyle{remark}
\usepackage{graphicx}
\usepackage{subfigure}
\usepackage{float} 
\usepackage{caption}
\usepackage[section]{placeins}
\usepackage{enumitem}
\usepackage{hyperref}
{
\hypersetup{hypertex=true,colorlinks=blue, anchorcolor=blue,citecolor=blue}
}

\newtheorem{theorem}{Theorem}
\newtheorem*{proof}{Proof}
\newtheorem{proposition}{Proposition} 
\usepackage{glossaries}
\newacronym{AI}{AI}{Artificial Intelligence}
\newacronym{LLMs}{LLMs}{Large Language Models}
\newacronym{MAS}{MAS}{Multi-Agent Systems}

\journal{Chinese Journal of Aeronautics}

\begin{document}

\begin{frontmatter}


\title{Enhancing Robustness of LLM-Driven Multi-Agent Systems through Randomized Smoothing}


\author[inst1]{Jinwei Hu}
\author[inst1]{Yi Dong\corref{cor1}}
\author[inst2]{Zhengtao Ding}
\cortext[cor1]{Corresponding author}
\author[inst1]{Xiaowei Huang}
\affiliation[inst1]{organization={Department of Computer Science}, 
            addressline={University of Liverpool}, 
            city={Liverpool},
            postcode={L69 3BX}, 
            country={UK}}
\affiliation[inst2]{organization={Department of Electrical and Electronic Engineering}, 
            addressline={The University of Manchester}, 
            city={Manchester},
            postcode={M13 9PL}, 
            country={UK}}



\begin{abstract}
    This paper presents a defense framework for enhancing the safety of large language model (LLM)-empowered multi-agent systems (MAS) in safety-critical domains such as aerospace. We apply Randomized Smoothing—a statistical robustness certification technique—to the MAS consensus context, enabling probabilistic guarantees on agent decisions under adversarial influence. Unlike traditional verification methods, our approach operates in black-box settings and employs a two-stage adaptive sampling mechanism to balance robustness and computational efficiency. Simulation results demonstrate that our method effectively prevents the propagation of adversarial behaviors and hallucinations while maintaining consensus performance. This work provides a practical and scalable path toward safe deployment of LLM-based MAS in real-world high-stakes environments.
\end{abstract}




\end{frontmatter}


\section{Introduction}
\gls{MAS} play a critical role in a broad spectrum of domains including aerospace applications, where they are increasingly employed for cooperative decision-making, autonomous flight coordination, and mission execution. Traditionally, \gls{MAS} in these contexts have relied on well-defined mathematical models and provable convergence properties to achieve reliable consensus and coordination among agents \cite{dong2025nash, chunyan2024multiple, cao2012overview, vrohidis2017safe}. Methods such as average consensus protocols \cite{1205192}, leader-follower schemes \cite{ding2015consensus}, and distributed optimization algorithms \cite{liu2017constrained} provide formal safety and performance guarantees \cite{amirkhani2022consensus}. Similarly, conventional safe planning methods incorporate barrier functions, reachability analysis, and model predictive control to ensure consistent operational safety \cite{10529583, 10374434, yu2021model}. Nevertheless, with the rapid advancements in \gls{AI} and the emergence of \gls{LLMs}, these technologies are increasingly integrated into \gls{MAS}, introducing new adaptive decision-making and natural language interaction capabilities that extend beyond the traditional approaches \cite{chen2023multi}.
The widespread use of AI-related technologies in the safety-critical contexts also brings significant risks, particularly as real-world deployments increasingly reveal safety concerns in \gls{AI}-driven autonomous agents \cite{luettig2024ml, perez2024artificial, jabbour2024generative}. Recent incidents involving autonomous vehicles \cite{kim2020research, ma2025position}, drones and robots \cite{natarajan2025artificial}, as well as decision support systems highlight the potential risks when \gls{AI} systems fail to behave as expected \cite{acharya2025agentic}. Aerospace domains are particularly susceptible to these risks, such as misclassification of critical objects in autonomous navigation systems leading to collision events and erroneous recommendations in flight control systems compromising operational safety \cite{alves2018considerations}. \gls{AI}-driven flight management systems have demonstrated vulnerability to adversarial inputs that could potentially compromise trajectory planning and collision avoidance mechanisms during critical flight phases \cite{tian2021adversarial, hickling2023robust}. Additionally, unguaranteed behaviors in \gls{AI}-based intelligent groups can result in unintended coordination patterns that deviate from mission requirements, especially when operating under limited communication conditions, potentially resulting in catastrophic consequences in environments where operational margins are minimal and failure costs are extreme \cite{national2018time}.

As one of the most prominent AI technologies, \gls{LLMs} have begun to serve as a more flexible choice in modern \gls{MAS} for informed decision-making. While LLMs themselves have incorporated techniques such as noise-based smoothing to enhance robustness against adversarial perturbations \cite{ji-etal-2024-advancing}, the deployment of responsible LLM-driven \gls{MAS} can introduce more distinctive vulnerabilities: hallucinations can propagate through agent networks, reinforcing false beliefs that compromise decision-making \cite{ji2023survey,dong2024safeguardinglargelanguagemodels}; the black-box nature of these models makes perfect safeguards or trust management nearly impossible \cite{pmlr-v235-dong24c,hu2025trustorientedadaptiveguardrailslarge}; and their non-robust characteristics create unpredictable behaviors under minor input variations \cite{yin2025taijitextualanchoringimmunizing}. Furthermore, emergent behaviors have been observed in LLM-driven agents, such as unintended collusion and coordination patterns that may optimize for objectives misaligned with safety requirements \cite{motwani2024secret, lin2024strategic}. Traditional defense mechanisms for \gls{MAS}, including formal verification \cite{kouvaros2019formal}, invariant enforcement \cite{hibbard2022guaranteeing}, and model checking \cite{huang2019reasoning}, are inadequate for LLM-based systems due to their inherent stochasticity and the difficulty in formally specifying behavioral boundaries. While conventional \gls{MAS} rely on provable guarantees and well-defined mathematical constraints, the probabilistic nature of LLM-driven \gls{MAS} makes it impossible to verify safety or performance properties with absolute certainty \cite{wang-etal-2024-rethinking-bounds}. Although Chen et al. \cite{chen2023multi} recently explored achieving consensus in LLM-driven \gls{MAS} via iterative negotiation, their approach faces significant limitations in safety-critical aerospace environments: (1) vulnerability to stealthy malicious agents and hallucinations; (2) lack of formal safety guarantees, with no provable convergence under adversarial or uncertain conditions; and (3) non-deterministic behavior that amplifies uncertainty in consensus. These challenges highlight the urgent need for robust defense mechanisms that provide probabilistic safety guarantees across the entire agent network, rather than focusing solely on the dependability of single agent, while preserving the benefits of LLM-driven agents in aerospace applications \cite{hu2025position}.

In this work, we address these challenges in LLM-driven MAS by
proposing a novel defense framework based on Randomized Smoothing, a statistical certification technique originally developed for enhancing the robustness of deep neural networks against adversarial perturbations \cite{cohen2019certified}. Our approach adapts this technique to the \gls{MAS} consensus context, providing probabilistic guarantees on the reliability of inter-agent communication. The key innovation lies in a two-stage adaptive sampling strategy that efficiently verifies agent states by injecting Gaussian noise before input processing, thereby creating a certified radius within which agent decisions remain stable despite potential adversarial manipulations. Unlike traditional verification techniques that require complete model transparency, our method operates in the black-box setting inherent to LLMs, making it particularly suitable for practical deployment. Additionally, we implement a dynamic certification mechanism that adjusts verification intensity based on first-stage variance measurements, optimizing the trade-off between computational efficiency and safety assurance. Extensive simulations demonstrate that our approach effectively mitigates malicious behaviors in consensus-seeking scenarios, preventing the propagation of hallucinations and adversarial influence through the agent network while maintaining system performance. 



\section{Problem Formulation and Preliminaries} \label{problem}

\subsection{Problem Setup}
Consider a multi-agent system (MAS) with $n$ agents, where $n \geq 2$, denoted as $\mathcal{A} = \{a_1, a_2, \ldots, a_n\}$, with each agent being driven by a large language model (LLM). The communication structure between these agents is represented by a directed graph $\mathcal{G} = (\mathcal{V}, \mathcal{E})$, where $\mathcal{V} = \{1, 2, \ldots, n\}$ corresponds to the set of agents, and an edge $(i, j) \in \mathcal{E}$ indicates that agent $a_i$ can receive information from agent $a_j$. Each agent $a_i$ maintains a state $x_i(t)$ at discrete time step $t$. Similar to the work of Chen et al. \cite{chen2023multi}, each agent is initially assigned a random numerical value between $0$ and $1$ as its state. This value can be interpreted as the agent's position in a one-dimensional space or as an abstraction representing more complex information such as opinions, solutions, or decisions. The decision-making process of each agent is facilitated by an LLM function $\Phi_i: \mathcal{X}^{|\mathcal{N}_i|+1} \to \mathcal{X}$, where $\mathcal{N}_i = \{j | (i,j) \in \mathcal{E}\}$ represents the set of neighbors from which agent $a_i$ can receive information. The state update equation is:
\begin{equation} \label{eq:state_update}
    x_i(t+1) = \Phi_i(x_i(t), \{x_j(t) | j \in \mathcal{N}_i\})
\end{equation}

Ideally, the collective objective is for these agents to reach a consensus through an iterative decision-making process, where consensus is defined as all agents converging to the common state value as $t \rightarrow \infty$:
\begin{equation} \label{eq:perfect consensus}
   \lim_{t \to \infty} \|x_i(t) - x_j(t)\| = 0, \quad \forall i,j \in \{1,2,\ldots,n\}
\end{equation}

\subsection{Threat Model} \label{threat model}
We consider two primary threats in LLM-driven MAS:

\subsubsection{Stealthy Malicious Agents} Our threat model considers malicious influence that may appear in two complementary forms. Let $\mathcal{A}_{\text{mal}} \subset \mathcal{A}$ represent a subset of agents subject to malicious behavior. Each affected agent $a_k \in \mathcal{A}_{\text{mal}}$ transmits manipulated information $\hat{x}_k(t)$ instead of its true state $x_k(t)$. This manipulation follows a bounded, potentially probabilistic pattern:
\begin{equation} \label{eq:malicious}
    \hat{x}_k(t) = 
    \begin{cases}
        x_k(t) & \text{with probability } 1-p_k(t) \\
        \tilde{x}_k(t) & \text{with probability } p_k(t)
    \end{cases}
\end{equation}
where $\tilde{x}_k(t)$ represents the deceptive information and bounded:
\begin{equation}
    \|\tilde{x}_k(t) - x_k(t)\| \leq \delta_{max}
\end{equation}

This model encompasses both internally compromised agents that intentionally transmit false information and legitimate agents whose communications are intercepted and modified by external attackers. In both scenarios, the malicious influence is characterized by its stealthy nature, using bounded perturbations $\delta$ and strategic probability $p_k(t)$ to disrupt the consensus process while potentially evading detection.

\subsubsection{LLM Hallucination} 
LLM-driven agent itself may generate hallucinated information with probability $p_h$, resulting in an unintended output:

\begin{equation} \label{eq:hallucination} 
    \tilde{x}_i(t+1) = 
    \begin{cases}
        \Phi_i(x_i(t), \{\hat{x}_j(t) | j \in \mathcal{N}_i\}) & \text{with probability } 1-p_h \\
        \Phi_i^{\text{hall}}(x_i(t), \{\hat{x}_j(t) | j \in \mathcal{N}_i\}) & \text{with probability } p_h
    \end{cases}
\end{equation} 
where $\hat{x}_j(t)$ represents the information received by agent $a_i$ from its neighbor $a_j$ (which might be true or manipulated), and $\Phi_i^{\text{hall}}$ represents the hallucination function that deviates from the intended reasoning path.

\section{Safe Consensus Seeking Methods} \label{twostage}
\subsection{Defense Methods for Robust Consensus}
In this section, we propose a defense framework to address the aforementioned challenges with randomized smoothing.

\subsubsection{Randomized Smoothing for LLMs}
For each agent $a_i$, we introduce a smoothing function $\mathcal{S}_i: \mathcal{X}^{|\mathcal{N}_i|+1} \to \mathcal{X}$ defined as:
\begin{equation} \label{eq:smooth function}
    \mathcal{S}_i(\mathbf{z}_i) = \mathbb{E}_{\boldsymbol{\epsilon} \sim \mathcal{N}(0, \sigma^2 \mathbf{I})}[\Phi_i(\mathbf{z}_i + \boldsymbol{\epsilon})]
\end{equation}
where $\mathbf{z}_i = [x_i(t), \{\hat{x}_j(t) | j \in \mathcal{N}_i\}]$ represents the input vector containing agent $a_i$'s own state $x_i(t)$ and the received states $\hat{x}_j(t)$ from all neighboring agents $j \in \mathcal{N}_i$. The term $\boldsymbol{\epsilon}$ is a noise vector drawn from a Gaussian distribution with variance $\sigma^2$. In our framework, randomized smoothing is applied at two levels:

1) When an agent $a_i$ needs information from a neighboring agent $a_j$, it does not accept the raw reported position directly. Instead, $a_i$ queries $a_j$ multiple times with slight perturbations to confirm consistency in $a_j$'s reported position. Each neighbor $a_j$ is queried repeatedly through its smoothing function $\mathcal{S}_j$ to generate a reliable state estimate $\hat{x}_j(t)$ and then share the verified state with $a_i$:
\begin{equation}
    \hat{x}_j(t) = \hat{\mathcal{S}}_j(\mathbf{z}_j)
\end{equation}

This repeated querying ensures that the reported position from each neighbor has been verified for consistency, filtering out potential hallucinations or malicious perturbations.

2) When agent $a_i$ makes its own decision based on these received (already smoothed) states, it applies an additional layer of randomized smoothing. Agent $a_i$ repeatedly queries its LLM function $\Phi_i$ with slightly perturbed versions of the entire input vector to obtain multiple samples of possible outputs. In practice, we approximate the expected value in Eq~\eqref{eq:smooth function} through Monte Carlo sampling:
\begin{equation}
    \hat{\mathcal{S}}_i(\mathbf{z}_i) = \text{trim-mean}\{\Phi_i(\mathbf{z}_i + \boldsymbol{\epsilon}_1), \Phi_i(\mathbf{z}_i + \boldsymbol{\epsilon}_2), \ldots, \Phi_i(\mathbf{z}_i + \boldsymbol{\epsilon}_m)\}
\end{equation}
where $\boldsymbol{\epsilon}_1, \boldsymbol{\epsilon}_2, \ldots, \boldsymbol{\epsilon}_m$ are $m$ independent samples from $\mathcal{N}(0, \sigma^2 \mathbf{I})$, and trim-mean represents the mean of the samples after removing the most extreme values. 

This dual-level smoothing approach provides enhanced robustness against both potential sources of perturbation: (1) it reduces the influence of adversarial manipulations or hallucinations in the information shared between agents, and (2) it stabilizes each agent's own state-updating process against inconsistencies in LLM outputs or perturbations in collective input.

\subsubsection{Adaptive Sampling Strategy}
To optimize computational resources while maintaining robust defense capabilities, we implement a adaptive sampling strategy for randomized smoothing. In the initial sampling stage, each agent performs a small number of initial samples $m_1$ to estimate the variance of the outputs:
\begin{equation}
   \mathbf{V}_i = \frac{1}{m_1} \sum_{k=1}^{m_1} \left\|\Phi_i(\mathbf{z} + \epsilon_k) - \frac{1}{m_1} \sum_{j=1}^{m_1} \Phi_i(\mathbf{z} + \epsilon_j)\right\|^2
\end{equation}
This initial variance estimate serves as a low-cost probe to detect potential inconsistencies in LLM agents' responses, which may indicate either adversarial manipulation or model hallucination. Following this, in the adaptive sampling stage, based on the observed variance $\mathbf{V}_i$, each agent adjusts the number of additional samples $m_2(i)$ according to:
\begin{equation}
   m_2(i) = \min\left\{\left\lceil c \cdot \frac{\mathbf{V}_i}{\tau} \right\rceil, m_{\max}\right\}
\end{equation}
where $c > 0$ is a scaling factor, $\tau > 0$ is a threshold parameter, and $m_{\max}$ is the maximum number of samples allowed for computational efficiency. Our implementation uses predefined variance thresholds to categorize response variability, allocating fewer samples when variance is low and more samples when variance is high. This adaptive approach efficiently balances computational resources with defensive robustness across varying threat scenarios.

\subsection{Probabilistic Safety Guarantees}
Our defense framework provides theoretical probabilistic guarantees on the robustness of the LLM agents' decision-making against adversarial manipulations and internal hallucinations.

\subsubsection{Theoretical Foundation for Continuous State Spaces}
We establish probabilistic robustness guarantees for LLM-driven multi-agent decision-making by extending Cohen et al.'s \cite{cohen2019certified} randomized smoothing framework from discrete classification to continuous state spaces in MAS. Building on the smoothed function $\mathcal{S}_i$ defined in Eq~\eqref{eq:smooth function}, we analyze how LLM agents maintain reliable decision-making when processing potentially perturbed information from neighboring agents.

In our MAS context, each LLM agent must make decisions based on a vector of inputs $\mathbf{z}$ that includes its own state and the reported states of neighboring agents. These reported states may contain adversarial perturbations or hallucinations. To analyze decision stability, we partition the agent's possible decision space into regions $\{R_1, R_2, ..., R_k\}$, representing different potential response strategies to the input information.

Let $p_A(\mathbf{z})$ be the probability that the LLM agent $a_i$ with input $\mathbf{z}$ produces a decision in region $R_A$ (its most likely decision region), and $p_B(\mathbf{z})$ be the probability it produces a decision in region $R_B$ (the second most likely region):
\begin{equation}
    p_A(\mathbf{z}) = \mathbb{P}(\Phi_i(\mathbf{z} + \epsilon) \in R_A) = \max_{j} \mathbb{P}(\Phi_i(\mathbf{z} + \epsilon) \in R_j)
\end{equation}
\begin{equation}
    p_B(\mathbf{z}) = \max_{j: R_j \neq R_A} \mathbb{P}(\Phi_i(\mathbf{z} + \epsilon) \in R_j)
\end{equation}

\begin{theorem}[Decision Stability Under Perturbed Information]
For an LLM-driven agent in a MAS using randomized smoothing with Gaussian noise $\mathcal{N}(0, \sigma^2I)$, if perturbations in the information received from neighbors are bounded by $\|\boldsymbol{\delta}\|_2$, where:
\begin{equation}
    \|\boldsymbol{\delta}\|_2 < \frac{\sigma}{2}(\Psi^{-1}(\underline{p_A}) - \Psi^{-1}(\overline{p_B}))
\end{equation}
then the agent's decision region remains stable despite the perturbed information, where $\underline{p_A}$ is a lower bound on the probability of the most likely decision region, $\overline{p_B}$ is an upper bound on the probability of the second most likely decision region, and $\Psi$ is the standard normal cumulative distribution function.
\end{theorem}

\begin{proof}
The key insight lies in quantifying how adversarial perturbations in the information exchanged between agents affect decision-making. Let $\mathbf{u = z + \boldsymbol{\delta}}$ represent the scenario where an agent receives perturbed information from its neighbors, with $\boldsymbol{\delta}$ representing the perturbation vector (either from malicious manipulation or hallucination). For any decision region $R$, the probability under perturbed information is:
\begin{align}
\mathbb{P}(\Phi_i(\mathbf{u} + \epsilon) \in R) = \mathbb{P}(\Phi_i(\mathbf{z} + \boldsymbol{\delta} + \epsilon) \in R)
\end{align}

Although our state space is continuous, we can reformulate our problem as a binary hypothesis testing problem to apply the Neyman-Pearson lemma. Specifically, for each decision region $R$, we define a binary case:
\begin{align}
h_R(\mathbf{x}) = 
\begin{cases} 
1, & \text{if}\ \Phi_i(\mathbf{x}) \in R \\
0, & \text{otherwise}
\end{cases}
\end{align}

This transforms our continuous state space problem into multiple binary classification problems, one for each decision region. For the most likely region $R_A$, we consider the binary hypothesis test:
\begin{align}
H_0 &: \mathbf{x} \sim \mathcal{N}(\mathbf{z}, \sigma^2 I) \\
H_1 &: \mathbf{x} \sim \mathcal{N}(\mathbf{z} + \boldsymbol{\delta}, \sigma^2 I)
\end{align}

Under the null hypothesis $H_0$, we have $\mathbb{P}(h_{R_A}(\mathbf{x}) = 1) = \mathbb{P}(\Phi_i(\mathbf{z} + \epsilon) \in R_A) = \underline{p_A}$. The Neyman-Pearson lemma helps us determine the most powerful test for detecting the shift from $H_0$ to $H_1$, and consequently provides bounds on how probabilities change under adversarial perturbations. By applying the Neyman-Pearson lemma \cite{neyman1933ix} and analyzing worst-case shifts in decision probabilities, we can establish the following bounds based on proof in \cite{cohen2019certified}:
\begin{align}
\mathbb{P}(\Phi_i(\mathbf{u} + \epsilon) \in R_A) &\geq \Psi(\Psi^{-1}(\underline{p_A}) - \|\boldsymbol{\delta}\|_2/\sigma) \\
\mathbb{P}(\Phi_i(\mathbf{u} + \epsilon) \in R_B) &\leq \Psi(\Psi^{-1}(\overline{p_B}) + \|\boldsymbol{\delta}\|_2/\sigma)
\end{align}

These bounds are derived from analyzing how the addition of the perturbation vector $\boldsymbol{\delta}$ affects the probability distribution of the agent's decisions. When considering the worst-case scenario, the perturbation vector is oriented to maximally decrease the probability of region $R_A$ and maximally increase the probability of region $R_B$. For the agent to maintain its original decision region despite receiving perturbed information, we require:
\begin{align}
\Psi(\Psi^{-1}(\underline{p_A}) - \|\boldsymbol{\delta}\|_2/\sigma) &> \Psi(\Psi^{-1}(\overline{p_B}) + \|\boldsymbol{\delta}\|_2/\sigma)
\end{align}

Since $\Psi$ is a strictly increasing function, this inequality is equivalent to:
\begin{align}
\Psi^{-1}(\underline{p_A}) - \|\boldsymbol{\delta}\|_2/\sigma &> \Psi^{-1}(\overline{p_B}) + \|\boldsymbol{\delta}\|_2/\sigma
\end{align}

Rearranging to isolate $\|\boldsymbol{\delta}\|_2$, we get:
\begin{align}
\|\boldsymbol{\delta}\|_2 < \frac{\sigma}{2}(\Psi^{-1}(\underline{p_A}) - \Psi^{-1}(\overline{p_B}))
\end{align}

This establishes our certification radius $r = \frac{\sigma}{2}(\Psi^{-1}(\underline{p_A}) - \Psi^{-1}(\overline{p_B}))$, providing a conservative guarantee on how much perturbation in the information received from neighbors an agent can tolerate while maintaining consistent decision-making.
\end{proof}

\subsection{From Local to Global Perspectives}
Having established that randomized smoothing stabilizes each agent's decisions under perturbations, we now extend these local guarantees to the global system. Our analysis demonstrates how individual robustness translates into network-wide resilience against both malicious attacks and hallucinations.

\subsubsection{Intuitive Understanding for Misinformation Mitigation}
The fundamental insight connecting local and global robustness lies in viewing both malicious misinformation and LLM hallucinations as sources of noise injected into the consensus process. Just as Gaussian perturbations are deliberately added to inputs during smoothing, a malicious behavior or a hallucinated output can be unitedly conceptualized as a "noise sample" that deviates from the truth. The key mechanism of randomized smoothing is that by aggregating multiple samples under controlled perturbations, an agent can differentiate between consistent signals and outliers.

When an agent is repeatedly queried its LLM-based policy under slight random perturbations, \textit{the law of large numbers ensures that the trimmed mean output converges to the true intended output, effectively filtering out inconsistent signals}. Consequently, even if an adversary injects false information or an agent momentarily hallucinates, a smoothing procedure with sufficient samples will dilute their influence. This stabilizes each agent's decision-making process—a single perturbed input becomes unlikely to significantly influence an agent's final action.

This local stability translates into global error mitigation: agent $a_i$ can identify both its own hallucinations and inconsistencies in its neighbors' reports through the smoothing mechanism, which effectively treats anomalous reports as noise to be filtered out. When faced with deceptive information from malicious agents, the repeated querying process of randomized smoothing significantly reduces the likelihood of successful manipulation, as false information would need to remain consistent across multiple perturbed queries to evade detection. Consequently, local smoothing at the individual agent level creates a network-wide dampening effect on error propagation, effectively preventing the cascade of misinformation throughout the multi-agent system.

\subsubsection{Unified View of Perturbations}
Mathematically, we formalize this intuition by treating perturbations—whether from malicious agents or hallucinations — as deviations from ideal input information. For agent $a_i$ receiving information at time $t$: $z_i(t) = [x_i(t), \{x_j(t) | j \in N_i\}]$. When perturbations exist, the actual received information becomes $z_i'(t) = z_i(t) + \boldsymbol{\delta}_i(t)$ where $\boldsymbol{\delta}_i(t)$ represents the perturbation vector from either source. This unified formulation facilitates addressing both threat categories through a single defensive framework.

\subsubsection{Network Perturbation Propagation and Attenuation}
\begin{proposition}[Network Perturbation Attenuation]
In multi-agent networks employing randomized smoothing defenses, information perturbations progressively attenuate during propagation through the network.
\end{proposition} 

Consider a perturbation propagation path: $a_1 \to a_2 \to ... \to a_k$. With initial perturbation $\boldsymbol{\delta}_1$ at agent $a_1$, randomized smoothing attenuates the output perturbation to:
\begin{equation}
\|\boldsymbol{\delta}_{out,1}\| \approx \|\boldsymbol{\delta}_1\| \cdot (1 - \Psi(\frac{r_1}{\sigma}))
\end{equation} 
The term $\Psi(\frac{r_1}{\sigma})$ quantifies the probability of perturbation classification and filtration by agent $a_1$, where $r_1$ denotes the certification radius and $\sigma$ represents the noise standard deviation; a higher value indicates enhanced capacity for detecting and eliminating adversarial disturbances. As this attenuated perturbation propagates through subsequent agents, the perturbation reaching the $k$-th agent becomes:
\begin{equation}
\|\boldsymbol{\delta}_{out,k}\| \approx \|\boldsymbol{\delta}_1\| \cdot \prod_{i=1}^{k} (1 - \Psi(\frac{r_i}{\sigma}))
\end{equation} 
Since each factor $(1 - \Psi(\frac{r_i}{\sigma})) < 1$, the perturbation magnitude undergoes exponential decay, effectively preventing error amplification and cascading effects throughout the network.

\subsubsection{Tolerance to Malicious Agents}
\begin{proposition}[Malicious Tolerance]
Randomized smoothing enhance the \textit{systemic tolerance} to malicious agents.
\end{proposition}
Given a subset of malicious agents within a system of $n$ agents, each introducing perturbation vector $\boldsymbol{\delta}_{mal}$, when non-malicious agents employ randomized smoothing, the system's tolerance to malicious agents exhibits proportionality to the certification radius:
\begin{equation}
\frac{|A_{\text{mal}}|}{|A|} \propto \frac{r_{\text{min}}}{\|\boldsymbol{\delta}_{mal}\|_{max}}
\end{equation}
where $r_{\text{min}}$ is the minimum certification radius in the system, and $\|\boldsymbol{\delta}_{mal}\|_{max}$ denotes the maximum malicious perturbation.

\section{Numerical Simulation} \label{simulation}
\begin{figure*}[htbp]
    \centering
    \subfigure[Ideal consensus with no malicious agents]{
        \includegraphics[width=0.32\textwidth]{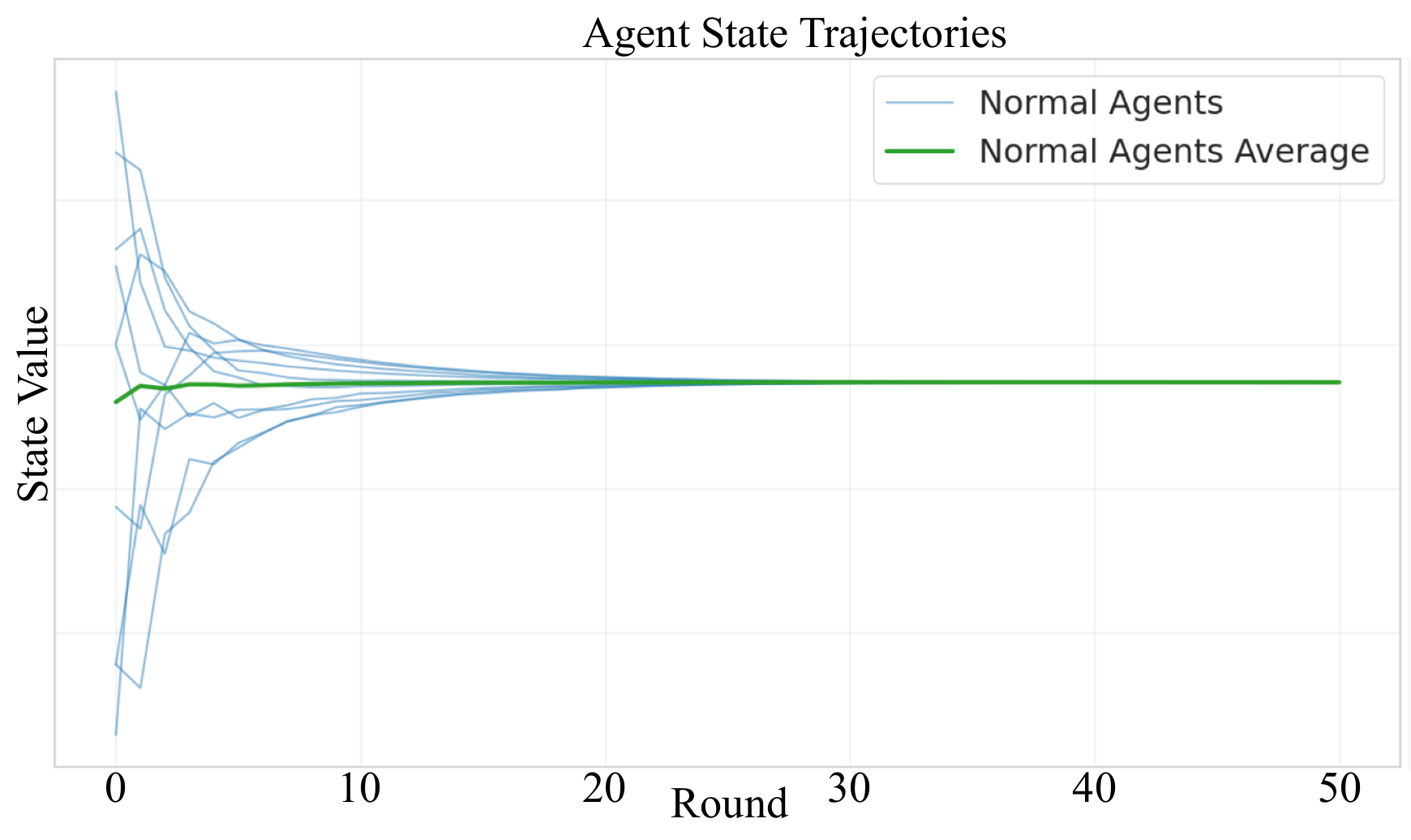} 
    }
    \subfigure[Malicious agents without defense mechanisms]{
        \includegraphics[width=0.32\textwidth]{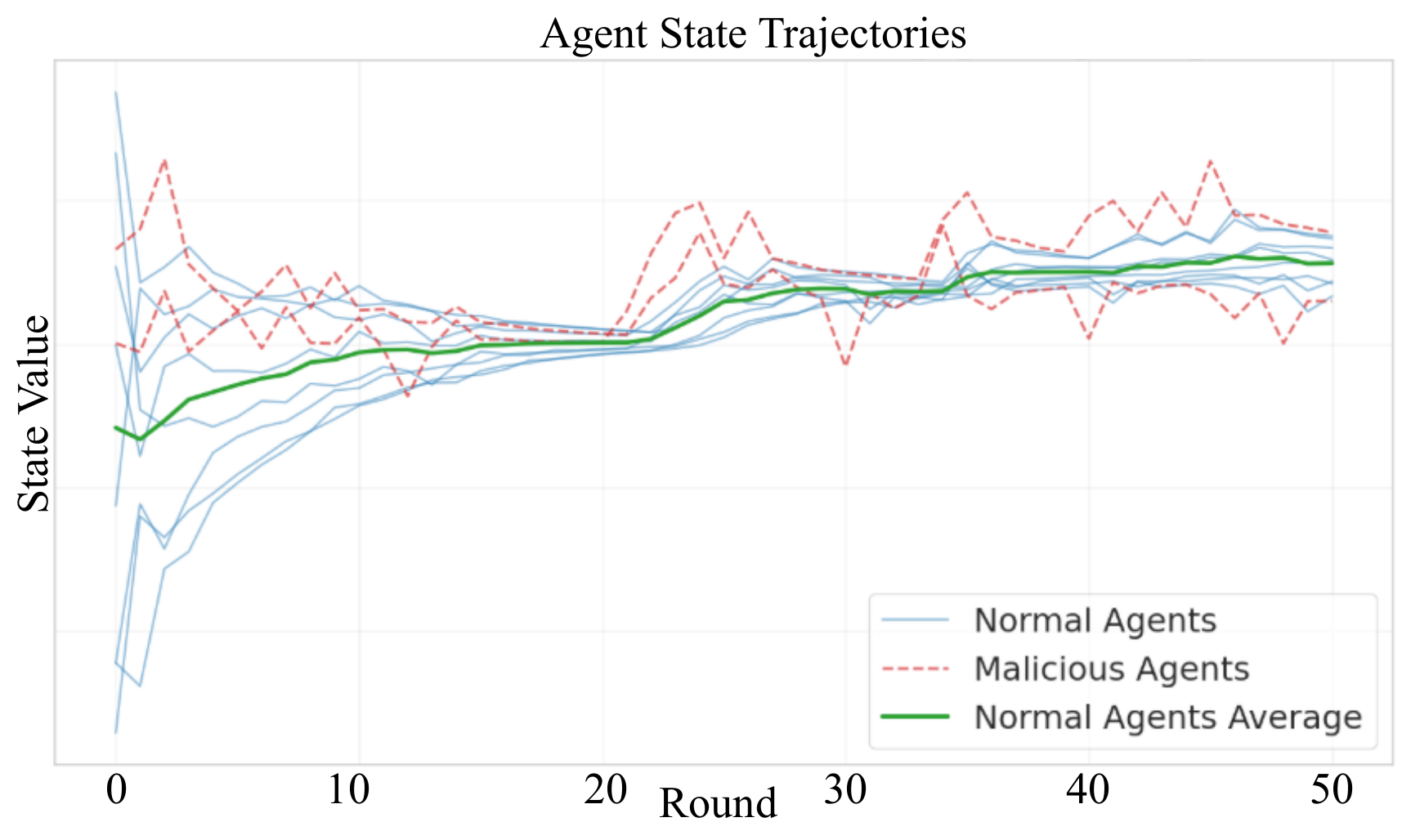}
    }
    \subfigure[Malicious agents present with defense]{
        \includegraphics[width=0.32\textwidth]{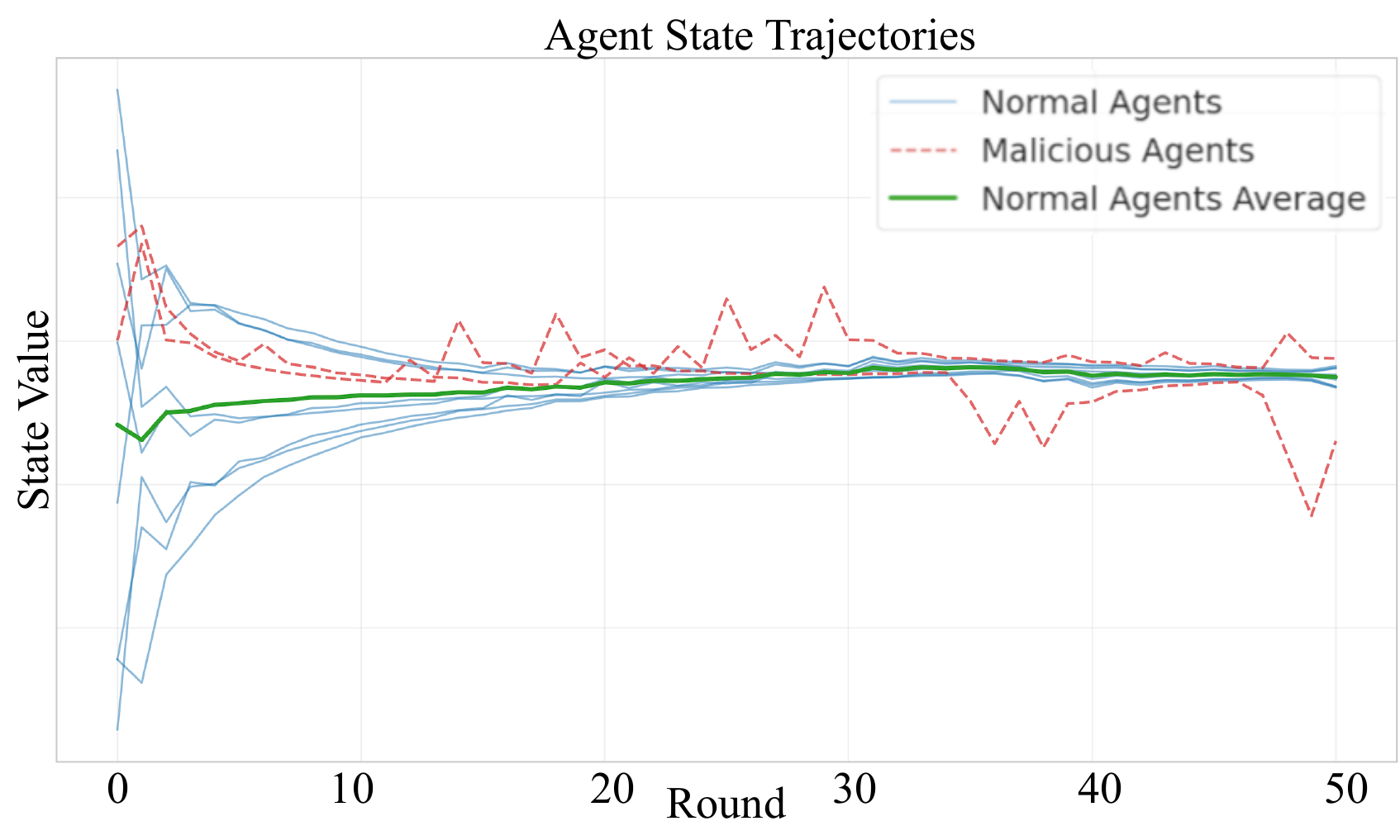}
    }
    \caption{Temporal Evolution of Agent State Trajectories under Multiple Scenarios}
    \label{fig:Agent State Trajectories}
\end{figure*}

\begin{figure*}[htbp]
    \centering
    \subfigure[]{
        \includegraphics[width=0.32\textwidth]{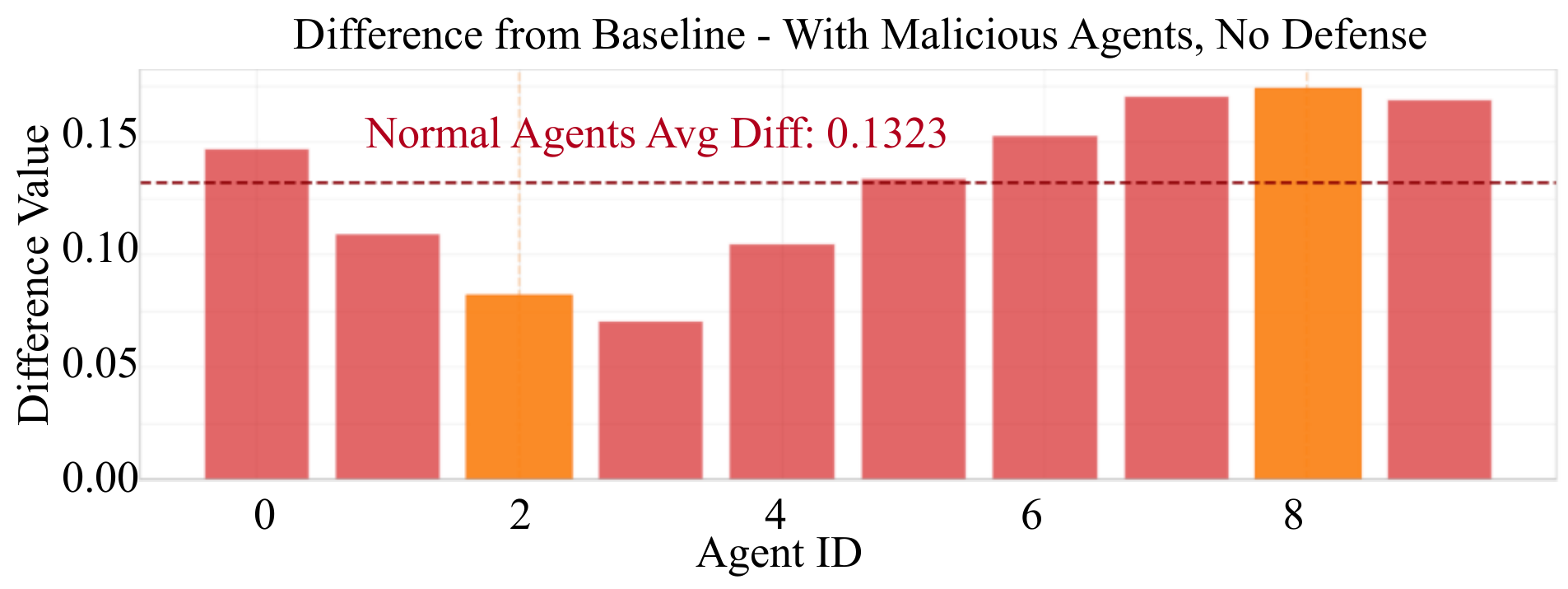} 
    }
    \subfigure[]{
        \includegraphics[width=0.32\textwidth]{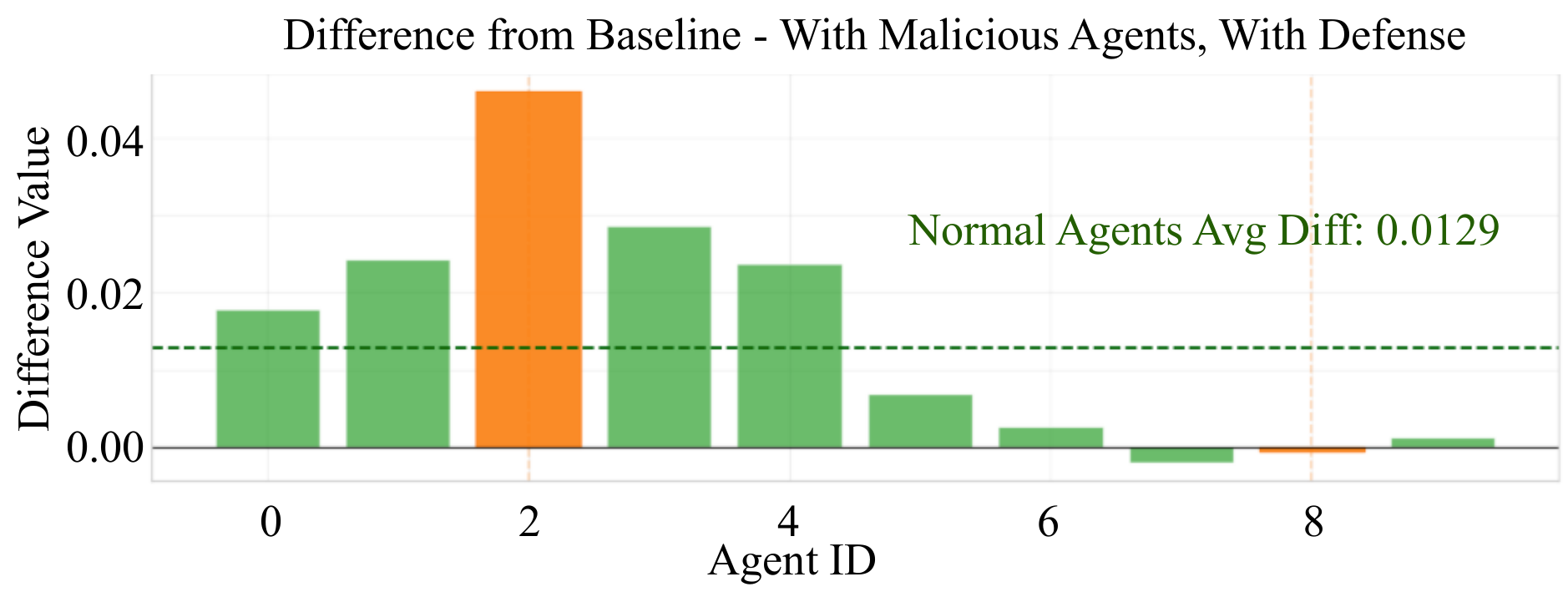}
    }
    \subfigure[]{
        \includegraphics[width=0.32\textwidth]{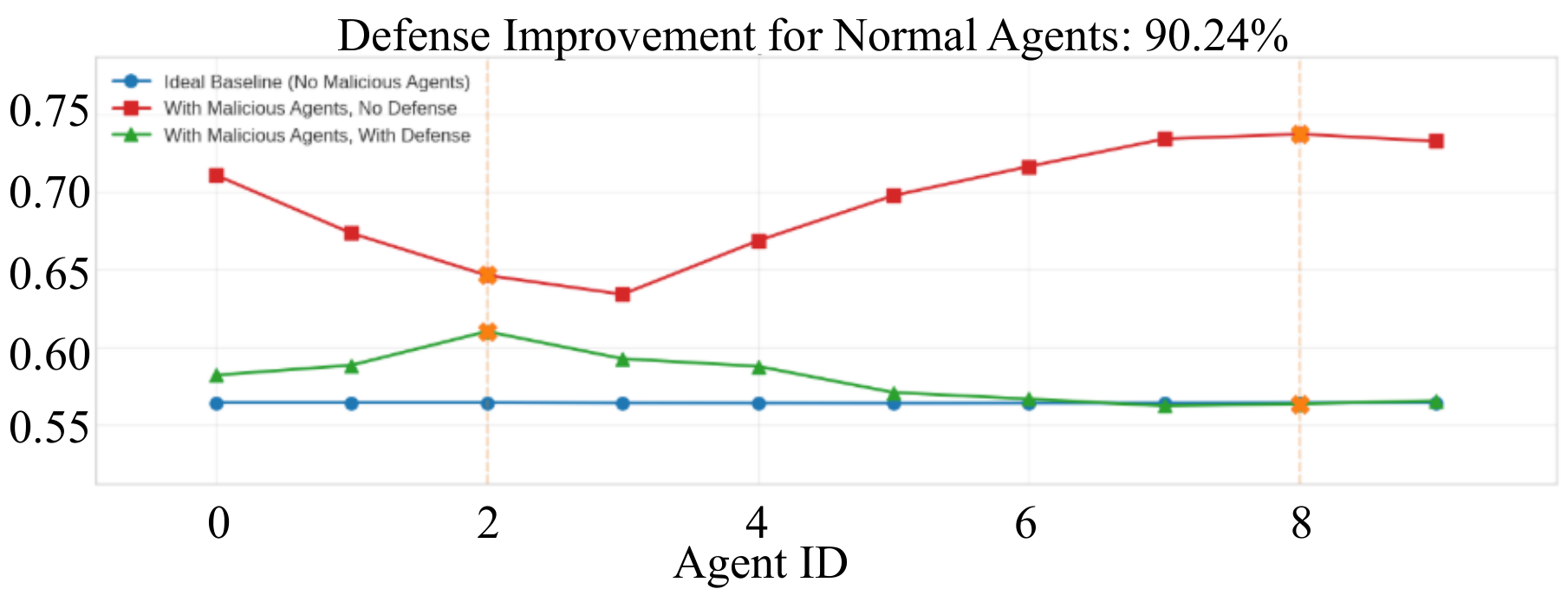}
    }
    \caption{Quantitative Assessment of Defense Performance Across LLM-Driven Agent Groups}
    \label{fig:Quantitative Assessment}
\end{figure*}
\subsection{Simulation Setup} \label{sec:Setup}
To evaluate our proposed randomized smoothing defense strategy in the context of cooperative decision and control of intelligent group systems, we implemented a ring-structured MAS with $n$ LLM-driven agents ($n = 10$ in our configuration), which provides a controlled environment to examine consensus-seeking behavior under adversarial conditions. In this topology, each agent $a_i$ communicate only with its adjacent neighbors, forming a directed graph that reflects the communication constraints typical of formation-flying UAVs and other distributed unmanned systems.

Each agent maintains a state $x_i(t) \in [0, 1]$ representing an abstraction of its position or decision value within the intelligent group. Agents are initialized with fixed values to ensure experimental reproducibility. The ideal consensus objective follows Eq.~(\ref{eq:perfect consensus}), requiring all agents in the intelligent group to converge to a common value as $t \rightarrow \infty$, which is essential for cooperative decision and control tasks.

The agent decision-making process is implemented using GPT-3.5-Turbo as the underlying LLM function $\Phi_i$, which facilitates computing state updates in accordance with Eq.~(\ref{eq:state_update}). This represents a novel approach to cooperative decision-making in intelligent group systems. Following our threat model formulation, we adopt a unified view of perturbations by treating malicious manipulations and hallucinations as manifestations of misinformation within the system. Specifically, we implement a misinformation probability that combines the effects of manipulated information from potentially compromised agents in $\mathcal{A}_{\text{mal}} \subset \mathcal{A}$ (as formalized in Eq.~(\ref{eq:malicious})) and the naturally occurring hallucinations inherent to LLM-driven processes is also kept (as defined in Eq.~(\ref{eq:hallucination})). The simulation operates with parallel agent updates to accurately capture the concurrent nature of distributed autonomous systems in intelligent group operations. 

\subsection{Performance Metrics}
To quantitatively evaluate the effectiveness of our randomized smoothing defense framework, we employ several complementary metrics that measure the system's resilience to malicious influences across multiple dimensions.

\subsubsection{State Deviation Metrics}
The primary assessment criterion is the deviation between agent states in different experimental scenarios for each agent $i$. 
\begin{equation}
\Delta_i^{\text{no-def}} = x_i^{\text{no-def}} - x_i^{\text{baseline}}
\end{equation}
\begin{equation}
\Delta_i^{\text{def}} = x_i^{\text{def}} - x_i^{\text{baseline}}
\end{equation}
where $x_i^{\text{no-def}}$ and $x_i^{\text{def}}$ represent the agent's final state in scenarios with malicious agents (without and with defense, respectively), and $x_i^{\text{baseline}}$ is the state in the ideal baseline scenario (no malicious agents). The absolute values of these deviations, $|\Delta_i^{\text{no-def}}|$ and $|\Delta_i^{\text{def}}|$, provide a measure of how far each agent deviates from ideal behavior.

\subsubsection{Normal Agent Performance}
Since our primary concern is protecting normal agents from malicious influences, we focus particularly on the average deviation across all normal agents:
\begin{equation}
\bar{\Delta}_{\text{normal}}^{\text{no-def}} = \frac{1}{|\mathcal{A}_{\text{normal}}|}\sum_{i \in \mathcal{A}_{\text{normal}}}|\Delta_i^{\text{no-def}}|
\end{equation}
\begin{equation}
\bar{\Delta}_{\text{normal}}^{\text{def}} = \frac{1}{|\mathcal{A}_{\text{normal}}|}\sum_{i \in \mathcal{A}_{\text{normal}}}|\Delta_i^{\text{def}}|
\end{equation}
where $\mathcal{A}_{\text{normal}}$ represents the set of normal agents. Lower values indicate better defense performance.

\subsubsection{Defense Improvement Percentage}
To quantify the benefit of our defense mechanism, we calculate the percentage reduction in deviation by the defense:
\begin{equation}
\text{Improvement (\%)} = \frac{\bar{\Delta}_{\text{normal}}^{\text{no-def}} - \bar{\Delta}_{\text{normal}}^{\text{def}}}{\bar{\Delta}_{\text{normal}}^{\text{no-def}}} \times 100\%
\end{equation}
This metric provides an intuitive measure of defense effectiveness, with higher values indicating greater improvement.


\subsection{Comparative Analysis of Defense Effectiveness}
This section evaluates the performance of our proposed randomized smoothing defense framework through comprehensive simulations across multiple scenarios. Following the experimental setup described in Section \ref{sec:Setup}, we establish three comparative scenarios: (1) an ideal baseline with all benign agents, (2) a \gls{MAS} system with malicious agents but no defense, and (3) the same \gls{MAS} system with malicious agents but with our proposed defense strategy.

Figure \ref{fig:Agent State Trajectories} illustrates the state trajectories of cooperative agents under three distinct scenarios. In scenario (a), all agents are benign, resulting in smooth convergence to consensus, confirming that LLM-driven agent groups can achieve excellent cooperative decision-making. Scenario (b) demonstrates how malicious agents (red) significantly disrupt this process, causing persistent oscillations and preventing a stable agreement, with normal agents' average state (green line) diverging considerably from the ideal baseline. Scenario (c) showcases our randomized smoothing defense mechanism's effectiveness, which successfully mitigates the influence of identical malicious agents, resulting in more stable trajectories with reduced oscillations and an average state closely approximating the ideal baseline and confirming the framework's capability to preserve cooperative decision-making integrity under adversarial conditions. 

The quantitative assessment presented in Figure \ref{fig:Quantitative Assessment} provides clear evidence of our defense mechanism's efficacy. Panel (a) illustrates the significant deviation of individual agents from the ideal baseline when malicious agents are present without defense measures, with normal agents exhibiting an average deviation of 0.1251 units. In stark contrast, panel (b) demonstrates how the implementation of randomized smoothing substantially reduces these deviations, with most agents showing markedly lower disparities from baseline states (average deviation among normal agents decreases to merely 0.0129 units). Panel (c) highlights the defense improvement metrics, revealing a 90.24\% average reduction in deviation across the entire intelligent agent group. These comparative results empirically validate our theoretical guarantees, confirming that the probabilistic certification provided by randomized smoothing effectively bounds the impact of perturbations in LLM-driven cooperative decision-making systems.

\subsection{Aerospace Simulation}
\begin{figure*}[htbp!]
    \centering
    \includegraphics[width=0.93\textwidth]{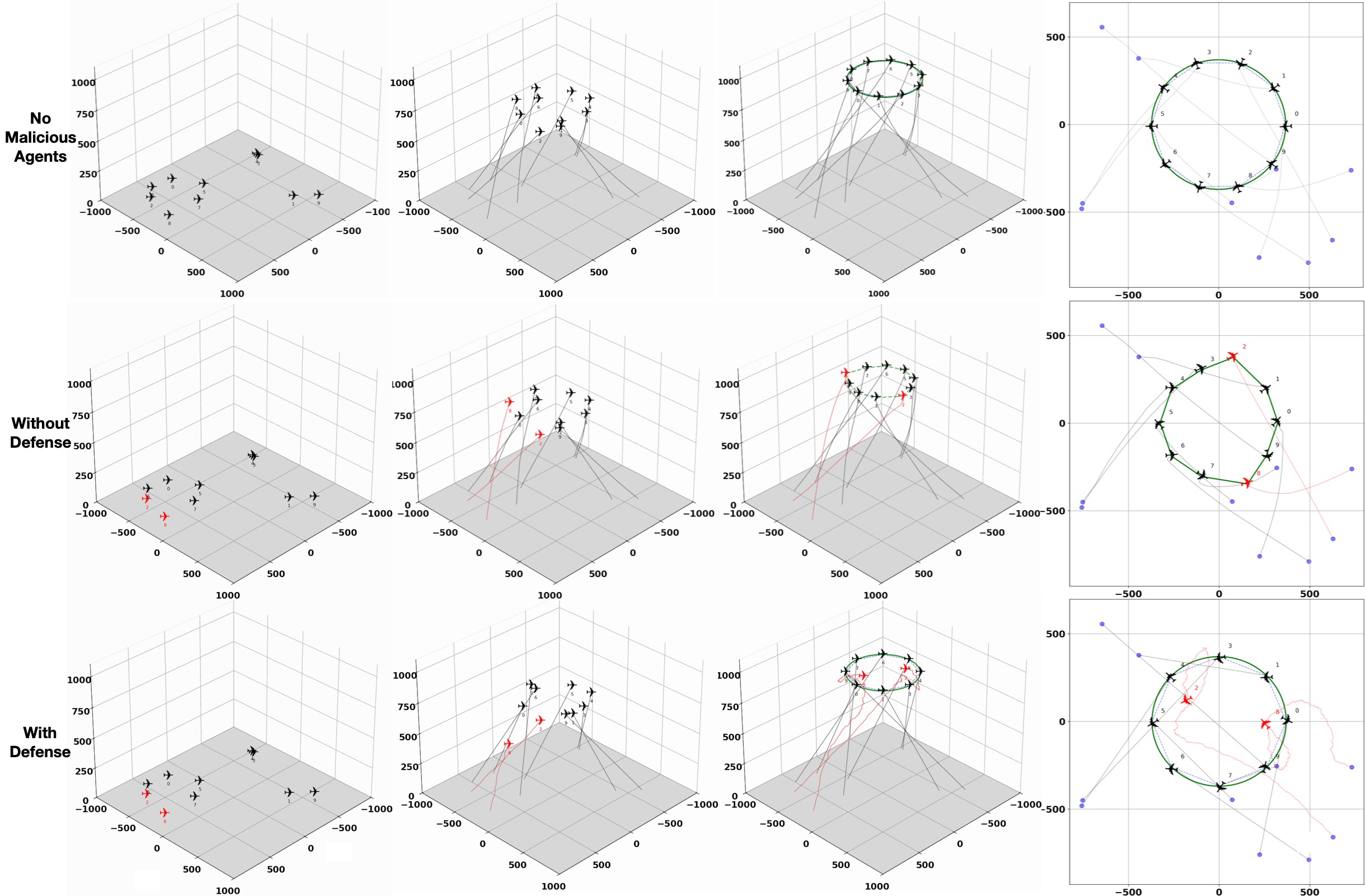}
    \caption{Temporal evolution of LLM-driven UAVs' state trajectories under adversarial scenarios}
    \label{fig:real-world simulation}
\end{figure*}
To validate our theoretical findings in a more realistic setting, we extend our evaluation to a three-dimensional aerospace scenario, wherein multiple LLM-driven UAVs must coordinate their flight trajectories while maintaining safe separation distances. Figure~\ref{fig:real-world simulation} visualizes agent behaviors under varying experimental conditions and presents a top-view of their final formation states. The simulation involves ten UAVs navigating a 2000×2000×1000 cubic meter airspace from randomly initialized positions toward a shared consensus objective, with the aim of maintaining a predefined geometric formation.

In the baseline scenario (top row), where all agents operate benignly without adversarial interference, the UAVs exhibit smooth trajectory convergence and consistently preserve the target formation. These results highlight the capacity of LLM-driven agents to achieve effective coordination in complex 3D environments under normal operational conditions. In contrast, the introduction of malicious agents without defense mechanisms (middle row) leads to severe formation degradation. Malicious UAVs, which are highlighted in red, emit deceptive positional information that induces neighboring agents to perform unnecessary avoidance maneuvers, resulting in erratic trajectories, increased separation distances, and overall mission failure due to disrupted consensus.

The deployment of our randomized smoothing defense (bottom row) significantly mitigates these adversarial effects. Despite the presence of identical malicious agents, the protected formation closely resembles the baseline behavior, maintaining stability and coherence throughout the mission. Our defense mechanism can effectively filter out inconsistent or adversarial signals, preventing their propagation across the agent network.

These results demonstrate that our framework enables robust coordination in realistic aerospace settings, ensuring formation integrity under adversarial perturbations. The proposed defense mechanism enhances the resilience and reliability of LLM-driven autonomous systems in safety-critical missions.

\subsection{Discussion on Real-time Feasibility}
\label{sec:discussion}

The computational efficiency of LLM-based MAS in safety-critical aerospace environments are crucial for practical deployment. In our experimental setting, we leverage the GPT API as the underlying LLM engine and utilize  parallel execution via the \texttt{nest\_asyncio} library to enable scalable consensus simulations. This parallelized execution ensures that the overall runtime of our approach is primarily governed by the latency of \textit{individual LLM queries}, which remains within practically acceptable ranges for aerospace mission scenarios.

From a theoretical perspective, the computational complexity of our defense framework scales linearly with the number of samples $m$ used for certification, i.e., $\mathcal{O}(m)$. The adaptive two-stage sampling mechanism dynamically determines this sample size based on the variance observed in the first stage, maintaining a balance between computational overhead and safety guarantees while minimizing the total number of LLM queries. This design prioritizes minimal query usage to meet real-time constraints: unnecessary queries are reduced in low-variance conditions, while robustness is ensured in high-variance scenarios. This dynamic adaptation provides a tunable mechanism to balance safety assurance and decision latency, accommodating the computational demands of diverse aerospace missions.

In practical deployment, the reliance on parallel LLM queries allows for efficient execution even when the number of agents increases, as the practical runtime is determined by the latency of a single LLM query rather than the cumulative number of queries. However, ensuring real-time performance in safety-critical aerospace systems extends beyond algorithmic complexity. The real-world deployment should consider hardware-in-the-loop validation, examining how different hardware platforms and network conditions impact the parallel query performance. This will further define the trade-off between sample size, safety certification, and real-time responsiveness to ensure reliable deployment in operational aerospace missions.

\section{Conclusion} \label{conclusion}
This paper introduced a novel defense framework using randomized smoothing to enhance the safety and robustness of LLM-driven MAS in safety-critical scenarios. We improved resilience against adversarial inputs and hallucination propagation, particularly in consensus-seeking tasks, and our simulations confirmed these expectations by demonstrating stable and reliable agent coordination under uncertainty. The proposed method offers a promising approach for deploying LLM-based MAS in real-world aerospace and autonomous systems where traditional safety guarantees fall short.


 \bibliographystyle{elsarticle-num} 
 \bibliography{cjoa}





\end{document}